\documentclass{article}
\usepackage[preprint]{neurips_2026}
\usepackage[utf8]{inputenc}
\usepackage[T1]{fontenc}
\usepackage[colorlinks=true, linkcolor=red, citecolor=blue, urlcolor=blue]{hyperref}
\usepackage{url}
\usepackage{booktabs}
\usepackage{amsfonts}
\usepackage{nicefrac} 
\usepackage{mathtools}
\usepackage{graphicx}
\usepackage{amsmath,amsthm, amssymb}
\usepackage{tikz}
\usetikzlibrary{patterns}
\usepackage[dvipsnames]{xcolor}
\usepackage{array}
\usepackage{caption}
\usepackage{xspace}
\newtheorem{definition}{Definition}
\newtheorem{theorem}{Theorem}
\newtheorem{lemma}{Lemma}

\usepackage{mathrsfs}
\usepackage[colorinlistoftodos]{todonotes}
\usepackage{microtype}
\usepackage{bm}
\usepackage{color}
\usepackage{multirow}
\usepackage{subcaption}
\usepackage{cleveref}

\newcommand{\vx}{\bm{x}}
\newcommand{\vth}{\bm{\theta}}
\newcommand{\vt}{\bm{\tau}}
\newcommand{\name}{\textit{selective composition}\xspace}

\title{Compositional Generalization in Autoregressive Models via Logit Composition}

\author{%
  \textbf{Aakash Kumar} \\
  \textbf{Maria Sofia Bucarelli} \\
  \textbf{Emanuele Natale} \\
  COATI, CNRS, Inria, I3S,\\
  Université Coté d’Azur, France\\
  \texttt{\{aakash.kumar, maria-sofia.bucarelli, emanuele.natale\}@inria.fr}
}

\begin{document}

\maketitle

\begin{abstract}
Composing autoregressive models remains a core challenge in understanding how large language models can combine behaviors or skills learned across tasks. We introduce a new and principled composition strategy for autoregressive systems, inspired by composition methods developed for diffusion models. Under a factorized-conditionals assumption, we show that the resulting composition is projective: each component model preserves control over its own designated subspace of the output distribution avoiding interference between models. This property is further preserved under smooth reparameterizations of the output space, yielding a feature-space theorem. Finally, we show that composition preserves length-generalizing behavior when the factorization assumptions and component guarantees hold uniformly at the target length. These results provide a principled understanding of when model composition and merging succeed in autoregressive systems and identify conditions under which their interactions remain stable.

\end{abstract}

\section{Introduction}
    \label{sec:intro}
    Composing specialized generative models is an appealing path toward building systems that are more modular, reusable, and controllable than a single monolithic model. If one model captures a particular concept, another captures a different one, and a third provides a shared background, a natural question is whether these components can be combined \emph{at inference time} to produce a model that behaves as though it had learned all concepts jointly.

For Transformer-based autoregressive models, existing approaches to this problem typically operate in either \emph{parameter space} or \emph{architecture space}. Parameter-space methods (such as weight interpolation~\cite{utans1996weight, wortsman2022model}, Task Arithmetic~\cite{ilharco2023editing}, and sparsity-based merging~\cite{yadav2023ties, yu2024language}) combine fine-tuned checkpoints by averaging or adding task vectors, but provide limited guarantees on the resulting output distribution. Architecture-space approaches, such as Mixture-of-Experts (MoE)~\cite{jacobs1991adaptive, shazeer2017outrageously, fedus2022switch}, combine models at the level of routing and activation, typically without a distinguished background model and with input-dependent selection. Together, these approaches do not provide a static, explicit composition rule at the level of output distributions with guarantees on the resulting autoregressive model.

A related line of work in diffusion models has demonstrated that models can be composed by \emph{linearly combining their score functions}~\cite{du2023reduce, liu2022compositional}. This approach enables compositional generation and can even extrapolate beyond the support of individual models. \cite{bradley2025mechanisms} provided a principled explanation of when such composition is valid, introducing the notion of \emph{projective composition} and identifying a factorization property as a key sufficient condition. However, this perspective has not been extended to autoregressive models, where generation is inherently sequential and conditioned on an evolving prefix.


Inspired by the projective composition framework of \cite{bradley2025mechanisms}, we extend their formalism to autoregressive models such as Transformers. Given a collection of diffusion models $(p_1, \dots, p_k, p_b)$, where $p_b$ denotes a shared background distribution, \cite{bradley2025mechanisms} studied the composition operator
\begin{equation}
\label{eq:com_op}
    \mathcal{C}[(p_1, p_2, \dots, p_k, p_b)] \equiv p_b \prod_{i=1}^{k} \frac{p_i}{p_b}.
\end{equation}
They formalized the notion of correct composition and showed that the operator in
Equation~\ref{eq:com_op} yields a correct composition under certain conditions. Our \name extends this probability-space composition principle to autoregressive models such as Transformers, which form the basis of large language models (LLMs), by applying an explicit arithmetic rule to conditional next-token distributions.

Beyond correctness, we establish a conditional length-generalization guarantee. In particular, when the factorization assumptions and component-model guarantees hold uniformly at a target length, the composed model inherits the corresponding behavior on the combined task. This highlights a key advantage of distribution-level composition: it preserves structural properties of the underlying models under explicit assumptions. To build intuition, we begin with a simple toy example illustrating how composition works in practice.

\subsection*{Example: compositional reasoning in trigonometry}
\label{sec:lettersubs}

Consider a collection of trigonometric problems of the form \textcolor{blue}{Prove $X = Y$}, where \textcolor{blue}{$X$} and \textcolor{blue}{$Y$} are trigonometric expressions. Suppose these problems can be decomposed into distinct skills (see Figure \ref{fig:trig}):
\begin{itemize}
    \item algebraic manipulation,
    \item application of the identity $\sin^2(x) + \cos^2(x) = 1$,
    \item application of the identity $\sec^2(x) - \tan^2(x) = 1$.
\end{itemize}

Let $p_b$ be a background Transformer that performs algebraic manipulations, while $p_1$ and $p_2$ are specialized models that apply the first and second identities, respectively. Each model is trained independently and excels only at its designated skill.

The key question is whether we can compose these models to get a model capable of solving problems that require \emph{multiple identities in sequence}. For example, can the composed model solve a problem that requires first rewriting an expression using $\sin^2(x) + \cos^2(x) = 1$ and later applying $\sec^2(x) - \tan^2(x) = 1$, interleaved with algebraic steps?

This setting highlights the central challenge of autoregressive composition: the applicability of each skill depends on the \emph{current prefix} of the generated solution. As we show in this work, under suitable conditions ensuring that these skills act on disjoint segments of the generation, the composed model successfully integrates all components and produces correct multi-step solutions.

\paragraph{Our Contributions}

\begin{itemize}
    \item We introduce a principled framework for composing autoregressive models in probability space, extending projective composition to the sequential setting.
    \item We identify conditions (factorized and joint factorizability) under which composition is provably correct.
    \item We show that composition preserves length-generalizing behavior under uniform factorization and component-model guarantees.
    \item We connect our framework to model merging methods, providing insight into when and why such techniques succeed.
    \item We empirically validate our approach on both synthetic and large language model benchmarks, showing exact behavior in the controlled setting and strong coding gains with retention of mathematical performance in LLM evaluations.

\end{itemize}
\paragraph{Paper organization:}
The rest of the paper is organized as follows. Section~\ref{sec:rel} reviews related work on model merging and mixture-of-experts methods. Section~\ref{sec:main_theory} introduces our compositional framework and presents our main theory. Section~\ref{sec:model_merging} further connects our framework to prior model merging approaches. Section~\ref{sec:expt} presents empirical results on both controlled toy tasks and large language model benchmarks. Finally, we conclude in Section~\ref{sec:conc}.

\section{Related Work}
    \label{sec:rel}
    \paragraph{Model Merging}
Recent research has explored how to combine pretrained models without joint retraining, a paradigm known as model merging or parameter arithmetic. Early approaches focused on simple weight-space interpolation between models fine-tuned on different tasks, showing that linear interpolation can sometimes yield models capable of retaining knowledge from both sources \cite{utans1996weight, wortsman2022model}. A key insight from this line is Task Arithmetic \cite{ilharco2023editing}, which introduced the concept of \textit{task vectors} $\tau_t = \theta(t) - \theta(0)$, where $\theta(t)$ denotes parameters fine-tuned on task $t$ and $\theta(0)$ is the pretrained model. Task vectors are interpreted as directions in parameter space that meaningfully steer model behavior, enabling arithmetic operations across fine-tuned models.

Subsequent work has focused on reducing interference between tasks when performing such parameter arithmetic, for example by choosing merging coefficients more effectively \cite{akiba, yangadamerging, matena2022merging, jindataless}. Another line of research uses sparsification to mitigate task interference \cite{yadav2023ties, yangadamerging, yu2024language, davari2024model}. Other work uses low-rank subspaces and matrix factorization to extract informative components for merging \cite{Gargiulo2025, marczakno, stoica2024model}. More recent geometric approaches perform fine-tuning and merging in the tangent space of the pretrained model manifold, linearizing updates to better capture local structure, amplify weight disentanglement, and mitigate task interference \cite{ortiz-jimenez2023task}.

Our work connects to the model merging literature because we also compose a pretrained base (background) model with task-specific fine-tuned models, similar to Task Arithmetic \cite{ilharco2023editing} and subsequent merging methods. However, unlike parameter-space or task-space techniques, our \name operates in \textit{probability space}: it explicitly combines output distributions of component models while preserving their designated subspaces. This makes the comparison to mixture-of-experts methods useful, although the mechanism is different.

\paragraph{Mixture of Experts in Large Language Models}

Classical Mixture-of-Experts (MoE) architectures augment neural networks with multiple experts and a routing mechanism that selects a subset of experts per input, improving performance and efficiency by allocating capacity adaptively rather than relying on a single monolithic model \cite{jacobs1991adaptive, jordan1994hierarchical, shazeer2017outrageously}. In these settings there is typically no distinguished base model; instead, all experts are trained jointly and symmetrically, and a gating network learns to route inputs to suitable experts to reduce error and computational cost \cite{fedus2022switch, du2022glam}.

In the context of model merging, this idea has inspired \emph{routing-based merging} methods, which replace static, input-agnostic aggregation with adaptive, sample-dependent combinations of models or modules. Instead of using a fixed merged model for all inputs, these approaches dynamically decide how to combine task-specific components at inference time \cite{he2023merging, tang2024merging, lu2024twin, crisostomi2025mass, suntransformer, belofsky2023token, chronopoulou2023adaptersoup, muqeeth2024learning}. 

Our \name differs from both classical MoE and routing-based merging in several key aspects. Unlike MoE's learned gating network or routing-based merging's input-conditioned coefficients, our composition is entirely static: all component models contribute simultaneously without dynamic selection.
Moreover, classical MoE selects entire expert subnetworks via routing; routing-based merging dynamically weights parameter subspaces. We instead perform explicit arithmetic combination of output probability distributions.
While MoE relies on routing to avoid interference, our framework gives conditions under which each component preserves control over its designated output subspace.

Finally, our \name shares conceptual similarity with diffusion model composition methods, where models are composed by linearly combining their scores \cite{du2023reduce,liu2022compositional}. The literature on composing Transformers through their logits is more limited. \cite{logit_1} propose DExperts, which steers a base model by adding and subtracting expert and anti-expert logits in a product-of-experts formulation, enabling controlled generation without retraining. Similarly, \cite{logit_2} introduce contrastive decoding, which selects tokens by maximizing the difference between an expert and a weaker “amateur” model's log-probabilities under plausibility constraints. These approaches demonstrate that simple logit-level composition can effectively modulate generation behavior, motivating our formulation.

\section{Theory: Composing Transformers to Generalize}
    \label{sec:main_theory}
    In this section, we describe our setup. We begin with a toy example that is closely related to the framework of \cite{bradley2025mechanisms}. We then formalize the example introduced in Section \ref{sec:intro}. Finally, we present a result analogous to that of \cite{bradley2025mechanisms} in an abstract feature space. We denote by $p(y_t \mid x, y_{1:t-1})$ the token-wise distribution of a Transformer, which is a conditional distribution over a vocabulary $\mathcal{X} \subseteq \mathbb{R}^n$, where $n$ is the embedding dimension. Consider a collection of transformer models $(p_b, p_1, \dots, p_k)$, where each $p_i$ represents a distinct concept, and $p_b$ represents a common background model. We assume common support, so the ratios below are well-defined whenever a component assigns positive probability. We define a composed transformer $\hat{p}$, said to be a \name of $(p_b,p_1,\dots,p_k)$, by
\begin{equation}
\label{eq:comp_op_auto}
    \hat{p}(y_t| x, y_{1:t-1})
    \coloneqq
    \frac{1}{Z_t(x,y_{<t})}
    p_b(y_t| x, y_{1:t-1})
    \prod_{i=1}^k
    \frac{p_i(y_t| x, y_{1:t-1})}
         {p_b(y_t| x, y_{1:t-1})},
\end{equation}
where $Z_t(x,y_{<t})$ is the normalizing constant over $\mathcal X$. We study the compositional properties of $\hat{p}$. We begin with a simple illustrative example.

\subsection{Composing logits under factorized conditionals}
\label{subsec:toy_ex}
We illustrate the idea with a very simple example of letter-replacement Transformers. Consider three Transformer models $p_1, p_2$, and $p_b$, where the tokenization maps each character to a single token. Suppose $p_b$ is trained to return the same string it was given as an input, i.e., an identity function. Suppose $p_1$ is trained to replace the character ``\textcolor{blue}{A}'' with ``\textcolor{blue}{B}'' in a given string, while leaving all other tokens unchanged. $p_1$ can be thought of as a fine-tuned version of the background $p_b$ that can perform one extra task. Similarly, $p_2$ is trained to replace ``\textcolor{blue}{C}'' with ``\textcolor{blue}{D}'' while leaving all other tokens unchanged.
\begin{align*}
    p_1:\text{A}\to \text{B} && &T_1(\text{SET\textcolor{red}{A}JKHSDJH}) = \text{SET\textcolor{red}{B}JKHSDJH},\\
    p_2:\text{C}\to \text{D} && &T_2(\text{SET\textcolor{red}{C}JKHSDJH}) = \text{SET\textcolor{red}{D}JKHSDJH}.
\end{align*}
If we compose the above Transformers according to Equation \ref{eq:comp_op_auto}, \textbf{will the composed model be able to perform both transformations simultaneously?} In other words, will the composed model replace each ``\textcolor{blue}{A}'' with ``\textcolor{blue}{B}'' and each ``\textcolor{blue}{C}'' with ``\textcolor{blue}{D}'' in a given string? As we will see, this is indeed the case. We now formalize this idea to prove a result similar to \cite{bradley2025mechanisms} for diffusion models.

The idea is to define formally the property that the models and task are such that each model has its own predefined \textbf{range of action} depending on the input, where the model acts independently of the previous tokens. This is captured by the following definition. Here $\mathcal{X}^\ell$ is the space of length-$\ell$ generations.

\begin{definition}[Factorized conditionals]
\label{def:fac_con}
The family of conditional distributions $(p_1,\dots,p_k,p_b)$ on $\mathcal{X}^\ell$ (conditioned on some input $x$) satisfies the \emph{factorized conditionals} property if there exists a partition of token indices of $\mathcal{X}^\ell$
\begin{equation*}
[\ell]=M_1 \uplus M_2 \uplus \cdots \uplus M_k \uplus M_b
\end{equation*}
such that the following conditions hold. For a generation $y=(y_1,\dots,y_\ell)\in\mathcal{X}^\ell$, write $y_M$ for the subsequence indexed by $M$.
Then, for every $i\in\{1,\dots,k\}$:
\begin{enumerate}
    \item $y_{M_i}$ and $y_{M_i^c}$ are independent under $p_i(\cdot| x)$.
    \item $y_{M_b},y_{M_1},\dots,y_{M_k}$ are mutually independent under $p_b(\cdot| x)$.
    \item $p_i(y_{M_i^c}| x)=p_b(y_{M_i^c}| x)$.
\end{enumerate}
Equivalently, the factorization may be written as
\begin{align}
\label{eq:fac_con_factor}
p_i(y| x) &= p_i(y_{M_i}| x)\;p_b(y_{M_i^c}| x), \\
p_b(y| x) &= p_b(y_{M_b}| x)\prod_{i=1}^k p_b(y_{M_i}| x).
\end{align}
\end{definition}
The idea of Definition \ref{def:fac_con} is illustrated in Figure \ref{fig:fac_con}. $p_1$ and $p_2$ have their regions of action, and the composition $\hat{p}$ performs both tasks together. An important question is when to call a composition correct. \cite{bradley2025mechanisms} introduced the notion of \emph{projective composition}, which we adapt for this particular setting.
\begin{definition}[Projective Composition]
\label{def:proj_comp}
    Given a collection of distributions $\{p_i\}$ along with associated “projection” functions $\{\Pi_i:\mathcal{X}^\ell\to\mathbb{R}^m\}$, we call a distribution $\hat{p}$ a projective composition if
    \begin{equation*}
        \forall i: \qquad  \Pi_i\#\hat{p} =\Pi_i\#p_i.
    \end{equation*}
    That is, when $\hat{p}$ is pushed forward by each $\Pi_i$, it yields marginals identical to those of $p_i$.
\end{definition}
Note that here, $\Pi_i$'s act on the entire generation in $\mathcal{X}^\ell$ and return some projection in some space $\mathbb{R}^m$.
Having defined all the tools, we now state our result illustrating that Equation \ref{eq:comp_op_auto} indeed produces a composed model that exhibits the expected behavior.
\begin{theorem}[Correctness of Composition]
\label{thm:main}
    Suppose a set of autoregressive distributions $(p_1,\dots,p_k,p_b)$ conditioned on some input variable $x$, satisfy Definition~\ref{def:fac_con}, with corresponding masks $\{M_i\}_{i=1}^k$ and $M_b$. Consider a sample $y\in \mathbb{R}^{n\times \ell}$ generated by the composed model $\hat{p}(y|x)$ as defined in Equation \ref{eq:comp_op_auto}.
    Then the distribution of sample $y$ generated by $\hat{p}$ is given by
    \begin{equation*}
        \hat{p}(y| x) = p_b(y_{M_b}|x)\prod_{i=1}^k p_i(y_{M_i}|x).
    \end{equation*}
    In particular, the marginal of $\hat p(\cdot|x)$ on $M_i$ equals the marginal of $p_i(\cdot|x)$ on $M_i$ for all $i$, and so $\hat{p}$ is a projective composition with respect to projections $\{\Pi_i(y)=y_{M_i}\}_{i=1}^k$, as per Definition~\ref{def:proj_comp}.
\end{theorem}
\begin{proof}[Proof sketch]
    By Definition~\ref{def:fac_con}, each $p_i$ coincides with $p_b$ off its mask $M_i$, so at any timestep $t$ exactly one mask (either some $M_j$ or $M_b$) is active. In the product defining $\hat p$, all likelihood ratios cancel except for the active mask, yielding $\hat p(y_t|x,y_{<t})= p_j(y_t |x,y_{<t})$ (or $p_b$ if $t\in M_b$). Hence sampling proceeds by drawing each token from the unique distribution assigned to its mask, implying independence across disjoint masks. Therefore $\hat p(y|x)=p_b(y_{M_b}|x)\prod_{i=1}^k p_i(y_{M_i}|x)$ and the marginals on each $M_i$ match $p_i$, establishing projective composition. For detailed proof, see Appendix \ref{app:main_th}.
\end{proof}
In other words, Theorem \ref{thm:main}  illustrates that whenever this independence among the tasks (Definition \ref{def:fac_con}) holds, we indeed compose correctly.

\subsection{Composing with dependence on the past}
Having illustrated the compositional construction in the simple case of letter substitution in Subsection~\ref{subsec:toy_ex}, we now turn to the motivating example presented in the Introduction. Consider a collection of trigonometric problems of the form \textcolor{blue}{Prove $X = Y$}, where \textcolor{blue}{$X$} and \textcolor{blue}{$Y$} denote trigonometric expressions, and the objective is to establish their equivalence.

Assume that such problems can be partitioned into the following categories:
\begin{enumerate}
    \item Problems solvable via elementary algebraic manipulations.
    \item Problems solvable using the identity $\sin^2(x) + \cos^2(x) = 1$, in conjunction with elementary algebraic manipulations.
    \item Problems solvable using the identity $\sec^2(x) - \tan^2(x) = 1$, in conjunction with elementary algebraic manipulations.
\end{enumerate}

Suppose that a model $p_b$ performs well on problems of type (1), while models $p_1$ and $p_2$ perform well on problems of types (2) and (3), respectively. The question of interest is to determine the conditions under which the composed model
\begin{equation*}
    \hat{p} = \frac{p_1 p_2}{p_b}
\end{equation*}
is effective at solving problems that require the simultaneous application of both identities $\sin^2(x) + \cos^2(x) = 1$ and $\sec^2(x) - \tan^2(x) = 1$. More generally, for $k$ distinct concepts, we seek to characterize when a model $\hat{p}$, defined as in Equation~\ref{eq:comp_op_auto}, is capable of jointly leveraging all relevant components. It is important to note that the theoretical framework developed in Subsection~\ref{subsec:toy_ex} is not directly applicable in this setting. In particular, the behavior of $p_1$ is history-dependent: the applicability of the identity $\sin^2(x) + \cos^2(x) = 1$ at any given step depends on the sequence of transformations that have already been carried out in the evolving proof. To formalize this notion in the present setting, we introduce a property that we term \emph{skill factorizability}. This notion is analogous to the property of factorized conditionals (Definition~\ref{def:fac_con}) introduced in Subsection~\ref{subsec:toy_ex}. From here on, a model $p$ on a task space $T$ is understood to be a model whose inputs lie in the set $T$.

\begin{definition}[Skill factorizability]
For an autoregressive model $p$, a prompt $x\in T$ and continuation $y_{1:\ell} \in \mathcal{X}^\ell$, define
\begin{equation*}
\mathcal L_p(y_{1:\ell} \mid x)
:= \prod_{t=1}^\ell p\!\left(y_t \mid x, y_{<t}\right),
\quad \text{where } y_{<t} = (y_1,\dots,y_{t-1}).
\end{equation*}

We say that a model $p$ (on some task $T$), is \emph{skill factorizable} with respect to a base model $p_b$ if for every prompt $x\in T$, there exist stopping times $\tau_1 \le \tau_2$, such that, for all $\ell$ and all $y_{1:\ell} \in \mathcal{X}^\ell$,
\begin{equation*}
\mathcal L_{p}(y_{1:\ell} \mid x)
=
\Bigg(\prod_{t=1}^{\tau_1} p_{b}(y_t \mid x, y_{<t})\Bigg)
\Bigg(\prod_{t=\tau_1+1}^{\tau_2} p(y_t \mid x, y_{<t})\Bigg)
\Bigg(\prod_{t=\tau_2+1}^{\ell} p_{b}(y_t \mid x, y_{<t})\Bigg),
\end{equation*}
where $\tau_1,\tau_2$ may depend on the realized prefix $y_{<t}$.
\end{definition}

We also impose the following assumption: at any generation step, the task to be performed is uniquely determined by the generated prefix. In particular, there are no prefixes for which multiple distinct skills are simultaneously applicable. For example, in solving a trigonometric problem, if an identity is to be applied at a given step, then it is uniquely determined by the current expression; there are no steps at which multiple distinct identities could be applied. We formalize this in Definition \ref{def:join}.

\begin{definition}[Joint factorizability]
\label{def:join}
Let $T$ be a task space and $(p_b,p_1,\dots,p_k)$ be a collection of autoregressive models. We say that they are \emph{jointly factorizable} on $T$ if, for every prompt $x \in T$, the following hold:
\begin{itemize}
    \item For each $i \in \{1,\dots,k\}$, the model $p_i$ is skill factorizable with respect to $p_b$ on task $T$, with associated stopping times $(\tau_1^i,\tau_2^i)$.
    \item The active index sets $I_i=\{\tau_1^i+1,\dots,\tau_2^i\}$ are pairwise disjoint, i.e.,
    \begin{equation*}
    I_i \cap I_j = \emptyset \quad \text{for all } i \neq j.
    \end{equation*}
\end{itemize}
\end{definition}
We now show that if the models $(p_b,p_1,\dots,p_k)$ satisfy the joint factorizability property (Definition \ref{def:join})  on some task $T$, then the composed model $\hat{p}$ defined by Equation \ref{eq:comp_op_auto} indeed composes the concepts correctly.

\begin{theorem}[Composition under Joint factorizability]
\label{thm:join}
If models $(p_b,p_1,\dots,p_k)$ satisfy joint factorizability on task $T$ with active index sets $I_i=\{\tau_1^i+1,\dots,\tau_2^i\}$, then the composed model $\hat p$ defined by Equation~\eqref{eq:comp_op_auto} satisfies
\begin{equation*}
\mathcal L_{\hat p}(y_{1:\ell}\mid x)
=
\Bigg(\prod_{t\notin \bigcup_{i=1}^k I_i}
p_b(y_t\mid x,y_{<t})\Bigg)
\Bigg(\prod_{i=1}^k \prod_{t\in I_i}
p_i(y_t\mid x,y_{<t})\Bigg).
\end{equation*}
In particular, $\hat p$ applies each skill-specific model $p_i$ on its corresponding interval and agrees with the base model $p_b$ elsewhere.
\end{theorem}

\begin{proof}[Proof sketch]
At each timestep $t$, the unnormalized composed score is
\begin{equation*}
\tilde p(y_t\mid x,y_{<t})
=
p_b(y_t\mid x,y_{<t})\prod_{i=1}^k \frac{p_i(y_t\mid x,y_{<t})}{p_b(y_t\mid x,y_{<t})}.
\end{equation*}
By joint factorizability, for any prefix $y_{<t}$ there is at most one index $i$ such that $t\in I_i$. For all other $j\neq i$, we have
\begin{equation*}
p_j(y_t\mid x,y_{<t}) = p_b(y_t\mid x,y_{<t}),
\end{equation*}
so all likelihood ratios cancel except possibly one. The resulting score is already normalized, and hence
\begin{equation*}
\hat p(y_t\mid x,y_{<t})
=
\begin{cases}
p_i(y_t\mid x,y_{<t}), & t\in I_i,\\
p_b(y_t\mid x,y_{<t}), & \text{otherwise}.
\end{cases}
\end{equation*}
Multiplying over $t$ yields the desired factorization of $\mathcal L_{\hat p}$, showing that each model $p_i$ is applied exactly on its interval and $p_b$ elsewhere. See Appendix \ref{app:main_th} for details.
\end{proof}

\subsection{Composing logits under factorized conditionals in feature space}
Inspired by the result of \cite{bradley2025mechanisms} on diffusion models in feature space, we now present a version of Theorem \ref{thm:main} in feature space. The idea is that the composition $\hat{p}$ (Equation \ref{eq:comp_op_auto}) will compose correctly if the factorized conditionals property (Definition \ref{def:fac_con}) holds in some feature space. This may correspond to composition of two writing styles, where it is not necessary that each model has its own domain of tokens, but such factorization may exist in some abstract feature space.

\begin{theorem}[Feature-space Composition]
\label{thm:main_2}
Given a collection of autoregressive models $(p_1,\dots,p_k,p_b)$ conditioned on an input variable $x$, suppose their sequence-level laws admit densities with respect to a common base measure and have common support. Suppose further that there exists an invertible $C^1$ map $\mathcal{D}$ on the feature space such that the pushforward distributions
\begin{equation*}
\{\mathcal{D}\#\mathcal{L}_{p_1},\dots,\mathcal{D}\#\mathcal{L}_{p_k},\mathcal{D}\#\mathcal{L}_{p_b}\}
\end{equation*}
satisfy Definition~\ref{def:fac_con} with corresponding partition $(M_1,\dots,M_k,M_b)$. Then the composition $\hat p$ defined in Equation~\ref{eq:comp_op_auto} satisfies
\begin{equation*}
\mathcal{D}\#\mathcal{L}_{\hat p}
=
(\mathcal{D}\#\mathcal{L}_{p_b})_{M_b}
\prod_{i=1}^k (\mathcal{D}\#\mathcal{L}_{p_i})_{M_i}.
\end{equation*}
Therefore, $\hat p$ is a projective composition of $(p_1,\dots,p_k,p_b)$ with respect to the projection functions $\Pi_i(y)=\mathcal{D}(y)_{M_i}$ for $i=1,\dots,k$.
\end{theorem}

Similar to \cite{bradley2025mechanisms}, we utilized the  1-homogeneous property of composition operator defined in Equation \ref{eq:com_op} to prove that it is independent of  parameterization. Theorem \ref{thm:main_2} follows directly as a consequence. See Appendix \ref{app:main_th} for details.

\subsection{Length Generalization}
Length generalization refers to the ability of Transformer models to successfully perform a given task on input sequences that are longer than those encountered during training. \cite{zhou2024what} conjectured that a Transformer can achieve length generalization for a particular algorithmic task if there exists a sufficiently simple implementation of that algorithm within the RASP (Restricted Access Sequence Processing) framework. Our results are stated for arbitrary $\ell$. Consequently, if the hypotheses of Theorems~\ref{thm:main}, \ref{thm:join}, or \ref{thm:main_2} hold uniformly at a target length and each component model performs its corresponding task at that length, then the composed model $\hat p$ inherits the corresponding behavior on the composed task. This is a conditional transfer statement: the theory does not by itself prove that arbitrary pretrained or fine-tuned Transformers length-generalize, nor that approximate factorization must persist outside the lengths on which the components were trained.

\section{Model Merging}
    \label{sec:model_merging}
    Model merging refers to the process of combining models trained on distinct tasks into a single model that achieves strong performance across all tasks, ideally matching or exceeding that of the individual task-specific experts. Various approaches to model merging have been proposed (see Section \ref{sec:rel}); however, the underlying reasons for their empirical success remain poorly understood. In this section, we present insights suggesting that our results may help explain why such methods are effective in practice, and how they relate to prior attempts at theoretical justification \citep{ortiz-jimenez2023task}. We begin with a simplified setting to illustrate the core intuition in Subsection \ref{subsec:ta_lin}, and subsequently extend the discussion to a more general framework in Subsection \ref{subsec:ta}. In particular, we consider \emph{task arithmetic}, a model merging technique defined as follows. Let $f_b:\mathbb{R}^n \to \mathbb{R}^m$ denote a base neural network trained on an initial task. This base model is then fine-tuned on multiple tasks, yielding a collection of expert models $f_1, f_2, \dots, f_k$. Let $\theta_b$ denote the parameter vector of $f_b$, and let $\theta_1, \theta_2, \dots, \theta_k$ denote the parameter vectors corresponding to $f_1, f_2, \dots, f_k$, respectively. Task arithmetic constructs a merged model $f_{\mathrm{merged}}$ with parameters given by
\begin{equation}
\label{eq:ta}
    \theta_{\mathrm{merged}} \coloneq  \theta_b + \sum_{i=1}^k (\theta_i - \theta_b).
\end{equation}
Note that the formula for merging weights in task arithmetic is the same as the formula for merging scores (or $\log$ probabilities of tokens) in our composition (Equation \ref{eq:comp_op_auto}). 
\subsection{Task Arithmetic: Linear Case}
\label{subsec:ta_lin}
We present the idea under the simplifying assumption that the models 
$f_b, f_1, \dots, f_k$ are linear maps from $\mathbb{R}^n$ to $\mathbb{R}^m$. 
In particular, for each $i$, $f_i(x) \in \mathbb{R}^m$. We assume that these models satisfy \emph{factorized conditionals}, which in this case we call Disjoint coordinate factorization, meaning that each model $f_i$ differs from the base model $f_b$ only on a disjoint subset of output coordinates.

\begin{definition}[Disjoint coordinate factorization]
\label{def:disjoint}
Let $f_b, f_1, \dots, f_k : \mathbb{R}^n \to \mathbb{R}^m$ be linear maps. 
We say they satisfy disjoint coordinate factorization if there exist masks 
$M_1, \dots, M_k \in \{0,1\}^m$ such that:

\begin{enumerate}
    \item (Disjointness of the masks) For all $i \neq j$ and all $r \in \{1,\dots,m\}$,
    \begin{equation*}
    M_i^{(r)} = 1 \;\Rightarrow\; M_j^{(r)} = 0.
    \end{equation*}

    \item (Localized deviation) For every $i \in \{1,\dots,k\}$ and all 
    $x \in \mathbb{R}^n$,
    \begin{equation*}
    (f_i(x))^{(r)} = (f_b(x))^{(r)} 
    \quad \text{ whenever }  M_i^{(r)} = 0.
    \end{equation*}
\end{enumerate}
Equivalently, each difference $f_i - f_b$ is supported only on the 
coordinates selected by $M_i$.
\end{definition}
If the models $f_b, f_1, \dots, f_k$ satisfy Definition~\ref{def:disjoint}, then the merged model $f_{\mathrm{merged}}$ defined in Equation~\ref{eq:ta} produces outputs that coincide with each expert model on its corresponding specialized coordinates.

This leads to the following key observation: \textbf{combining weights according to task arithmetic is equivalent to combining outputs according to our composition}, at least in the linear setting.

\subsection{Task Arithmetic: General Case}
\label{subsec:ta}

\cite{ortiz-jimenez2023task} show that task arithmetic can succeed when each task vector affects only its own region of the input space. More precisely, let $\mathcal{T}=\{\vt_t\}_{t\in[T]}$ be a collection of task vectors, and let $\mathcal{D}=\{\mathcal{D}_t\subset \mathcal{X}\}_{t\in[T]}$ be a collection of pairwise disjoint task supports, that is $
\mathcal{D}_t \cap \mathcal{D}_{t'}=\varnothing $ {for } $t\neq t'$.
We say that a network $f$ satisfies the task arithmetic property around $\vth_0$ with respect to $\mathcal{T}$ and $\mathcal{D}$ if, for every choice of coefficients $(\alpha_1,\dots,\alpha_T)\in \mathcal{A}\subseteq \mathbb{R}^T$,
\begin{equation*}
f\!\left(\vx;\vth_0+\sum_{t=1}^T \alpha_t\,\vt_t\right)
=
\begin{cases}
f(\vx;\vth_0+\alpha_t\,\vt_t), & \vx\in\mathcal{D}_t,\\
f(\vx;\vth_0), & \vx\notin\bigcup_{t=1}^T \mathcal{D}_t.
\end{cases}
\end{equation*}

In other words, when the input belongs to the support of task $t$, the merged model behaves exactly as if only the corresponding task vector $\vt_t$ had been added; on inputs outside all task supports, it reduces to the base model. This is the nonlinear analogue of the disjoint-coordinate condition discussed in Subsection~\ref{subsec:ta_lin}: in both cases, the different task updates do not interfere with one another. Thus, successful task arithmetic can be interpreted as a parameter-space realization of the same additive composition principle that our method applies directly to conditional output distributions: each component contributes only on its own region of action, while the base model governs the remaining behavior.

\section{Experiments}
    \label{sec:expt}
    We empirically validate the proposed framework in two stages. First, we revisit the controlled toy setting from Subsection~\ref{subsec:toy_ex} to verify that the composition rule behaves as intended in a fully transparent example. We observe that our \name performs perfectly on this task. The details of this experiment are provided in Appendix~\ref{app:expt}. Second, we evaluate the method on downstream language tasks by composing the base model $p_{\text{base}}$ given by \texttt{Gemma 2 2B} \citep{gemma2} and its two task-specific fine-tuned versions from \texttt{MergeBench} \citep{mergebench}, one specialized for mathematics ($p_{\text{math}}$) and one specialized for coding ($p_{\text{coding}}$). We then measure the performance of the resulting model $\hat{p}$ obtained by composing $(p_{\text{base}},p_{\text{math}},p_{\text{coding}})$ via Equation~\ref{eq:comp_op_auto}, on four benchmarks: two coding tasks (\texttt{HumanEval+} \citep{HumanEval} and \texttt{MBPP+} \citep{MBPP}) and two mathematics tasks (\texttt{GSM8K} \citep{GSM8k} and \texttt{MATH} \citep{MATH}), and compare it to the performance of the baselines $(p_{\text{base}},p_{\text{math}}, p_{\text{coding}})$. The results are reported in Table~\ref{tab:expt_2}. Additional details about the evaluation protocol are provided in Appendix~\ref{app:expt}. Overall, these results show that \name can combine complementary fine-tuned models: it improves over the coding expert on both coding benchmarks while retaining much of the math expert's performance on mathematical benchmarks. We additionally experimented with composing more than two expert models, but observed a sharp degradation in generation quality due to instability around end-of-sequence behavior; further discussion is provided in Appendix~\ref{app:expt}.

\begin{table}[h]
\centering
\caption{Performance of $\hat{p}$ on downstream benchmarks compared to baselines.}
\label{tab:expt_2}
\begin{tabular}{lcccc}
\hline
\textbf{Benchmark} & $p_{\text{base}}$ & $p_{\text{coding}}$ & $p_{\text{math}}$ & $\hat{p}$\\
\hline
\texttt{GSM8K}     &  3.2  &  3.2  &  \textbf{50.2} & 41.2 \\
\texttt{MATH}      &  16.3  & 16.7   & \textbf{25.5}   & 24.2 \\
\texttt{HumanEval+}& 3.7   &  24.4  & 13.4   & \textbf{30.5} \\
\texttt{MBPP+}     &  33.9  & 38.1   & 32.8   & \textbf{39.4} \\
\hline
\end{tabular}
\end{table}

\paragraph{Hardware and code availability.}
Experiments were conducted on two NVIDIA H100 GPUs with 94 GB of memory each. The implementation is provided in the supplementary material to support reproducibility during the review process, and the code is available at \href{https://github.com/Aakash-verse/Compositional-Generalization-in-Autoregressive-Models-via-Logit-Composition}{GitHub repository}.

\section{Conclusion}
    \label{sec:conc}
    We introduced a probability-space framework for composing autoregressive Transformer models, extending the notion of projective composition beyond diffusion models to sequential generation. Our approach provides a principled way to combine specialized models with a shared background model via a simple multiplicative rule over conditional distributions. We established conditions under which this composition is correct, both in token-space and in a more general feature-space formulation, and showed that these guarantees extend to settings with prefix-dependent behavior. Importantly, we showed that composition preserves length-generalizing behavior when the factorization assumptions and component-model guarantees hold uniformly at the target length. Further, we connected our framework to existing model merging methods and showed that several empirical heuristics can be explained in terms of our composition principle. Finally, experiments on synthetic tasks and large language model benchmarks show exact behavior in the controlled setting and evidence of complementary skill transfer, with coding gains and retention of mathematical performance. Overall, this work suggests that autoregressive model composition can be understood as a structured probabilistic operation with strong theoretical guarantees, opening the door to more modular and controllable construction of large language models.

\paragraph{Limitations.}
To explain the compositional generalization abilities of our \name, we rely on factorization assumptions that may not fully hold in practice. While these assumptions help clarify why the composition succeeds when they are satisfied, they do not explain why pretraining—a process that may introduce interference—would lead to representations where such assumptions are approximately valid. In addition, the empirical evaluation is limited in scope, relying on a single base model and lacking strong comparisons with model merging alternatives.


\bibliographystyle{plainnat}
\bibliography{References}
\newpage


\appendix

\section{Notation}
    \label{app:notation}
    Table~\ref{tab:notation} summarizes the main notation used in Section~\ref{sec:main_theory} and in the proofs of \Cref{app:main_th}.

\begin{table}[h]
\centering
\small
\renewcommand{\arraystretch}{1.08}
\begin{tabular}{@{}p{0.25\linewidth}p{0.65\linewidth}@{}}
\toprule
Symbol & Meaning \\
\midrule
$\mathcal X$ & Token vocabulary; $\mathcal X^\ell$ is the space of length-$\ell$ generations. \\
$x$ & Conditioning input or prompt. \\
$y_{1:\ell}$, $y_t$, $y_{<t}$ & A generated sequence, its token at position $t$, and the prefix before $t$. \\
$p_b$ & Background or base autoregressive model. \\
$p_i$ & The $i$th specialized autoregressive model, for $i\in\{1,\dots,k\}$. \\
$\hat p$ & Composed autoregressive model defined by Equation~\ref{eq:comp_op_auto}. \\
$Z_t(x,y_{<t})$ & Normalizing constant for the composed conditional distribution at step $t$. \\
$\mathcal L_p(y_{1:\ell}\mid x)$ & Sequence-level likelihood induced by an autoregressive model $p$. \\
$M_i$, $M_b$ & Token-index masks partitioning $[\ell]$ into expert and background coordinates. \\
$y_{M_i}$ & Subsequence of $y$ restricted to indices in $M_i$. \\
$\Pi_i$ & Projection map used to define projective composition. \\
$\Pi_i\#p$ & Pushforward of distribution $p$ through projection $\Pi_i$. \\
$\tau_1^i,\tau_2^i$ & Stopping times delimiting the active region of skill $i$. \\
$I_i$ & Active interval $\{\tau_1^i+1,\dots,\tau_2^i\}$ for skill $i$. \\
$\mathcal D$ & Invertible feature map used in the feature-space theorem. \\
$\mathcal C[\cdot]$ & Distribution-level composition operator used in Appendix~\ref{app:main_th}. \\
\bottomrule
\end{tabular}
\caption{Notation used in the theoretical results and proofs.}
\label{tab:notation}
\end{table}

\section{Main Theory}
    \label{app:main_th}
    
In this section, we provide rigorous proof of all the results in the paper. We start with Theorem \ref{thm:main}.

\begin{proof}[Proof of Theorem \ref{thm:main}]
Fix the conditioning input $x$. For each position $t\in[\ell]$, let $m(t)\in\{b,1,\dots,k\}$ denote the unique index such that $t\in M_{m(t)}$. Since $\{M_1,\dots,M_k,M_b\}$ is a partition of $[\ell]$, such an index is well defined.

By Definition~\ref{def:fac_con}, for every $i\in\{1,\dots,k\}$, the distribution $p_i$ agrees with $p_b$ on all coordinates outside $M_i$. Equivalently, for every $t\notin M_i$,
\begin{equation*}
p_i(y_t \mid x, y_{1:t-1}) = p_b(y_t \mid x, y_{1:t-1}).
\end{equation*}
Therefore, if $t\in M_j$ for some $j\in\{1,\dots,k\}$, then
\begin{equation*}
\frac{p_i(y_t \mid x, y_{1:t-1})}{p_b(y_t \mid x, y_{1:t-1})}
=
\begin{cases}
1, & i\neq j,\\
\dfrac{p_j(y_t \mid x, y_{1:t-1})}{p_b(y_t \mid x, y_{1:t-1})}, & i=j.
\end{cases}
\end{equation*}
Hence the unnormalized composition score simplifies to
\begin{align*}
\tilde p(y_t \mid x, y_{1:t-1})
&=
p_b(y_t \mid x, y_{1:t-1})
\prod_{i=1}^k
\frac{p_i(y_t \mid x, y_{1:t-1})}{p_b(y_t \mid x, y_{1:t-1})}
\\&=
\begin{cases}
p_j(y_t \mid x, y_{1:t-1}), & t\in M_j,\ j\in\{1,\dots,k\},\\
p_b(y_t \mid x, y_{1:t-1}), & t\in M_b.
\end{cases}
\end{align*}
The right-hand side is already a normalized conditional distribution over the token vocabulary, so the normalizing constant in Equation~\ref{eq:comp_op_auto} equals one at this prefix.

Thus, at each position $t$, the composed model uses exactly the conditional distribution associated with the unique block containing $t$.

Now apply the chain rule:
\begin{equation*}
\hat p(y\mid x)
=
\prod_{t=1}^{\ell}\hat p(y_t\mid x,y_{1:t-1}).
\end{equation*}
Substituting the expression above yields
\begin{equation*}
\hat p(y\mid x)
=
\Bigg(\prod_{t\in M_b} p_b(y_t\mid x,y_{1:t-1})\Bigg)
\prod_{i=1}^k
\Bigg(\prod_{t\in M_i} p_i(y_t\mid x,y_{1:t-1})\Bigg).
\end{equation*}
By the factorized conditionals property, this equals
\begin{equation*}
\hat p(y\mid x)=p_b(y_{M_b}\mid x)\prod_{i=1}^k p_i(y_{M_i}\mid x).
\end{equation*}
In particular, for each $i$, the marginal of $\hat p$ on $M_i$ coincides with $p_i(\cdot\mid x)$. Therefore $\hat p$ is a projective composition with respect to the projections $\Pi_i(y)=y_{M_i}$, as per Definition~\ref{def:proj_comp}.
\end{proof}

We now come to Theorem \ref{thm:join}.

\begin{proof}[Proof of Theorem \ref{thm:join}]
Fix an arbitrary prompt $x\in T$ and an arbitrary continuation $y_{1:\ell}\in\mathcal X^\ell$. 
For notational simplicity, write
\begin{equation*}
p_{i,t} := p_i(y_t\mid x,y_{<t}),
\qquad
p_{b,t} := p_b(y_t\mid x,y_{<t}),
\qquad
\hat p_t := \hat p(y_t\mid x,y_{<t}).
\end{equation*}
Assume that the ratios $p_{i,t}/p_{b,t}$ are well-defined on the common support.

By definition of the composed model, before normalization,
\begin{equation*}
\tilde p_t
=
p_{b,t}\prod_{i=1}^k \frac{p_{i,t}}{p_{b,t}}.
\end{equation*}

Let
\begin{equation*}
I_i := \{\tau_1^i+1,\tau_1^i+2,\dots,\tau_2^i\}
\end{equation*}
denote the active interval of skill $i$. Since the intervals $\{I_i\}_{i=1}^k$ are pairwise disjoint, for every $t$ there is at most one index $i$ such that $t\in I_i$.

If $t\in I_i$ for some $i$, then for every $j\neq i$, we are outside the active interval of $p_j$, hence
\begin{equation*}
p_{j,t}=p_{b,t}.
\end{equation*}
Therefore,
\begin{equation*}
\tilde p_t
=
p_{b,t}\left(\frac{p_{i,t}}{p_{b,t}}\right)\prod_{j\neq i}\frac{p_{j,t}}{p_{b,t}}
=
p_{b,t}\left(\frac{p_{i,t}}{p_{b,t}}\right)
=
p_{i,t}.
\end{equation*}

If instead $t\notin \bigcup_{i=1}^k I_i$, then for all $i$ we have $p_{i,t}=p_{b,t}$, and hence
\begin{equation*}
\tilde p_t
=
p_{b,t}\prod_{i=1}^k \frac{p_{b,t}}{p_{b,t}}
=
p_{b,t}.
\end{equation*}

In both cases the unnormalized score is already a normalized conditional distribution, so the normalizing constant in Equation~\ref{eq:comp_op_auto} equals one. Thus, for every $t\in\{1,\dots,\ell\}$,
\begin{equation*}
\hat p(y_t\mid x,y_{<t})
=
\begin{cases}
p_i(y_t\mid x,y_{<t}), & t\in I_i \text{ for some } i,\\
p_b(y_t\mid x,y_{<t}), & t\notin \bigcup_{i=1}^k I_i.
\end{cases}
\end{equation*}

Multiplying over $t=1,\dots,\ell$ yields
\begin{equation*}
\mathcal L_{\hat p}(y_{1:\ell}\mid x)
=
\prod_{t=1}^\ell \hat p(y_t\mid x,y_{<t})
=
\Bigg(\prod_{t\notin \cup_{i=1}^k I_i} p_b(y_t\mid x,y_{<t})\Bigg)
\Bigg(\prod_{i=1}^k \prod_{t\in I_i} p_i(y_t\mid x,y_{<t})\Bigg).
\end{equation*}

This is exactly the claimed factorization, and shows that $\hat p$ applies each skill-specific model $p_i$ on its interval and $p_b$ elsewhere.
\end{proof}

We now come to the results in the abstract space. \cite{bradley2025mechanisms} utilized the  1-homogeneous property of composition operator defined in Equation \ref{eq:com_op} to prove that it is independent of  parameterization for distributions with densities on a common support. In particular, they proved the following lemma.
\begin{lemma}
\label{lem:par_ind}
    The composition operator defined in Equation \ref{eq:com_op} is independent of  parameterization, i.e.,
    \begin{equation*}
        \mathcal{C}[\mathcal{D}\#p_1,\dots,\mathcal{D}\#p_k,\mathcal{D}\#p_b] = 
        \mathcal{D}\#\mathcal{C}[p_1,\dots,p_k,p_b].
    \end{equation*}
\end{lemma}
Note that by a simple application of the chain rule, our composition (Equation \ref{eq:comp_op_auto}) simplifies to
\begin{equation*}
    \hat{p}(y_{1:T}|x) = p_b(y_{1:T}|x)\prod_{i=1}^k\frac{p_i(y_{1:T}|x)}{p_b(y_{1:T}|x)}.
\end{equation*}
Hence for the Proof of Theorem~\ref{thm:main_2}, we will speak in terms of distributions on the entire generation $y_{1:T}$, rather than distributions over the next token. We may simply use $p_i$ for $\mathcal{L}_{p_i}$ (or $p_i(y_{1:T}|x)$) for convenience.
\begin{proof}[Proof of Theorem~\ref{thm:main_2}]
Let
\begin{equation*}
q_i \coloneqq \mathcal{D}\#\mathcal{L}_{p_i}
\qquad\text{for } i=1,\dots,k,
\qquad\text{and}\qquad
q_b \coloneqq \mathcal{D}\#\mathcal{L}_{p_b}.
\end{equation*}
By assumption, the family $(q_1,\dots,q_k,q_b)$ satisfies Definition~\ref{def:fac_con}. Therefore, Theorem~\ref{thm:main} applied in the transformed feature space gives
\begin{equation*}
\mathcal{C}[q_1,\dots,q_k,q_b]
=
(q_b)_{M_b}\prod_{i=1}^k (q_i)_{M_i}.
\end{equation*}

By Lemma~\ref{lem:par_ind},
\begin{equation*}
\mathcal{C}[q_1,\dots,q_k,q_b]
=
\mathcal{C}[\mathcal{D}\#p_1,\dots,\mathcal{D}\#p_k,\mathcal{D}\#p_b]
=
\mathcal{D}\#\mathcal{C}[p_1,\dots,p_k,p_b].
\end{equation*}
Since $\hat p=\mathcal{C}[p_1,\dots,p_k,p_b]$, we obtain
\begin{equation*}
\mathcal{D}\#\mathcal{L}_{\hat p}
=
(\mathcal{D}\#\mathcal{L}_{p_b})_{M_b}
\prod_{i=1}^k (\mathcal{D}\#\mathcal{L}_{p_i})_{M_i}.
\end{equation*}
This is exactly the desired feature-space factorization.

Finally, since the transformed distributions satisfy the factorized conditionals property, the same argument as in Theorem~\ref{thm:main} shows that the projections $\Pi_i(y)=\mathcal{D}(y)_{M_i}$ recover the corresponding marginals. Hence $\hat p$ is a projective composition with respect to these projections.
\end{proof}

\section{Experiments}
    \label{app:expt}
    \subsection{Letter Replacement}

We begin with the toy example from Subsection~\ref{subsec:toy_ex}. We train a small decoder-only Transformer (GPT-2 style) from scratch using \texttt{GPT2LMHeadModel} \citep{gpt2, transformers_LIB} on a synthetic letter-replacement task. Each token represents a single character.

The vocabulary consists of the 26 uppercase English letters plus a space token. Each synthetic input has length 16. For each of the three tasks below, the submitted notebook constructs 2,000 training examples by sampling random strings and applying the task-specific deterministic transformation. The model uses \texttt{GPT2Config} with \texttt{n\_positions=32}, \texttt{n\_embd=64}, \texttt{n\_layer=2}, and \texttt{n\_head=2}. Models are trained for 5 epochs with AdamW, learning rate $3\cdot 10^{-4}$, and token-level cross-entropy loss. Decoding is greedy, taking the argmax token at each position.

We construct three models:
\begin{itemize}
    \item $p_b$: the base model, which implements the identity mapping (i.e., it returns the input string unchanged),
    \item $p_1$: a model that replaces every occurrence of \texttt{A} with \texttt{K},
    \item $p_2$: a model that replaces every occurrence of \texttt{M} with \texttt{B}.
\end{itemize}

We then compose these models using the rule in Equation~\ref{eq:comp_op_auto}. In the implementation, this corresponds to adding the log-probabilities of the two experts and subtracting the log-probabilities of the identity model. The resulting model $\hat{p}$ successfully combines both transformations: for any input string, it replaces every \texttt{A} with \texttt{K} and every \texttt{M} with \texttt{B}. In this toy setting, the method achieves a $100\%$ success rate over random unseen strings, confirming that the proposed composition operator captures the intended behavior.

\subsection{Experiments with LLMs}

We now provide further details of our experiments with LLMs in Section \ref{sec:expt}. We consider three Hugging Face checkpoints: \texttt{google/gemma-2-2b} for $p_{\text{base}}$, \texttt{MergeBench/gemma-2-2b\_math} for $p_{\text{math}}$, and \texttt{MergeBench/gemma-2-2b\_coding} for $p_{\text{coding}}$. The MergeBench checkpoints are domain-specialized variants derived from \texttt{Gemma2 2B}. The submitted supplementary notebook loads the models with \texttt{AutoModelForCausalLM} and \texttt{AutoTokenizer}; Hugging Face access is supplied through an external \texttt{HF\_TOKEN} environment variable, not hard-coded in the notebook.

We combine $p_{\text{base}}, p_{\text{math}},$ and $p_{\text{coding}}$ using our proposed \name operator (Equation~\ref{eq:comp_op_auto}) to produce a merged model $\hat{p}$. The implementation asserts that all tokenizers have the same vocabulary size and token-id mapping. During generation it stores separate key-value caches for the base, math, and coding models, computes
\begin{equation*}
\log \hat p_t = \log p_{\text{math},t}+\log p_{\text{coding},t}-\log p_{\text{base},t},
\end{equation*}
renormalizes with a softmax, and decodes greedily. We use \texttt{MAX\_NEW=512}, \texttt{SAMPLE\_MERGED=False}, \texttt{torch.manual\_seed(0)}, float16 on CUDA and float32 otherwise, and \texttt{device\_map="auto"} when CUDA is available. The base log-probabilities are clamped below at $-10^4$ for numerical stability.

\subsubsection{Evaluation Protocol}

\textbf{Mathematics.} Performance on mathematical reasoning is measured on two benchmarks: \texttt{GSM8K} and \texttt{MATH}. For \texttt{GSM8K}, the notebook calls \texttt{lm\_eval.simple\_evaluate} \citep{lm_eval_harness} with task \texttt{gsm8k}, \texttt{num\_fewshot=8}, and \texttt{batch\_size=1}; we report the flexible exact-match score. For the \texttt{MATH} benchmark, the notebook uses \texttt{DigitalLearningGmbH/MATH-lighteval}, configuration \texttt{all}, split \texttt{test}. It prepends a two-shot prompt and asks models to put the final answer in \texttt{\textbackslash boxed\{\}}. Correctness is determined by exact agreement of the extracted boxed answer after light normalization.

\textbf{Coding.} For code generation, we evaluate on \texttt{HumanEval+} and \texttt{MBPP+} using \texttt{evalplus} \citep{evalplus}. The notebook obtains problems via \texttt{get\_human\_eval\_plus} and \texttt{get\_mbpp\_plus}, writes generated completions to \texttt{samples\_humaneval.jsonl} and \texttt{samples\_mbpp.jsonl}, and then runs \texttt{evalplus.evaluate} once with dataset \texttt{humaneval} and once with dataset \texttt{mbpp}. Performance is measured using \texttt{pass@1}, which considers whether the first generated solution passes all tests. Generated programs are executed against the provided unit tests to determine correctness; for reproduction, code execution should be run in an isolated environment as recommended by \texttt{evalplus}.

\paragraph{Compute and supplementary material.}
Experiments were run on two NVIDIA H100 GPUs with 94 GB memory each. The supplementary material includes notebooks and README with the dependency installation command, model identifiers, evaluation commands, decoding settings, and the expected result table. Exact wall-clock times were not logged for every benchmark run. The code is available at \href{https://github.com/Aakash-verse/Compositional-Generalization-in-Autoregressive-Models-via-Logit-Composition}{GitHub repository}.

\subsubsection{Failure Mode with Additional Experts}

We also experimented with extending the composition operator to more than two specialized models by adding additional experts to the merge. In practice, performance dropped very sharply once more models were included. The merged model often became unstable and continued generating gibberish even after producing what looked like a complete code solution. A likely explanation is disagreement among the models about the end-of-sequence token and termination behavior. When the experts do not agree on when generation should stop, their combined log-probabilities can keep assigning nontrivial mass to continuation tokens, causing the merged model to over-generate instead of cleanly stopping at the end of the answer. This suggests that reliable multi-model composition may require stronger alignment of stopping criteria or an explicit mechanism for handling termination.

\section{Figures}
    \label{app:figs}
    \vspace{0.5cm}
\begin{figure}[h]
    \centering
    \input{Figures/trig_prob}
    \caption{A trigonometry problem requiring the use of basic algebraic manipulation and the identities \textcolor{red}{$\sin^2(x) + \cos^2(x) = 1$} and \textcolor{blue}{$\sec^2(x) - \tan^2(x) = 1$}.}
    \label{fig:trig}
\end{figure}
\vspace{0.5cm}
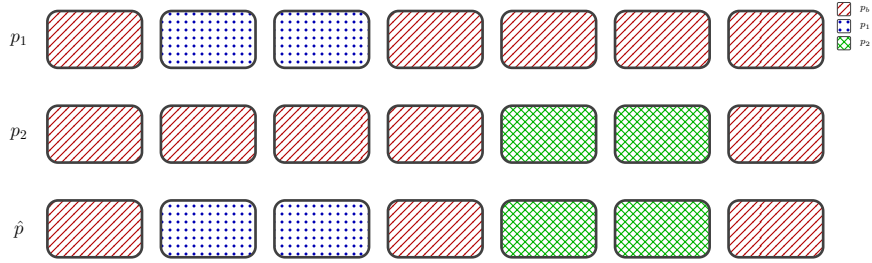
\begin{figure}[h]
    \centering
    \begin{tikzpicture}[scale=0.5, transform shape]

\tikzset{
    box/.style = {
        draw=darkgray,
        rounded corners,
        minimum width=2.5cm,
        minimum height=1.5cm,
        line width=1pt
    },
    redbox/.style = {
        box,
        fill=red!20,
        pattern=north east lines,
        pattern color=red!70!black
    },
    bluebox/.style = {
        box,
        fill=blue!20,
        pattern=dots,
        pattern color=blue!70!black
    },
    greenbox/.style = {
        box,
        fill=green!20,
        pattern=crosshatch,
        pattern color=green!70!black
    },
    legendbox/.style = {
        draw=darkgray,
        rounded corners,
        minimum width=0.55cm,
        minimum height=0.35cm,
        line width=1pt
    }
}

\node at (-2,0) {\Large $p_1$};
\node[redbox]  at (0,0) {};
\node[bluebox] at (3,0) {};
\node[bluebox] at (6,0) {};
\node[redbox]  at (9,0) {};
\node[redbox]  at (12,0) {};
\node[redbox]  at (15,0) {};
\node[redbox]  at (18,0) {};

\node at (-2,-2.5) {\Large $p_2$};
\node[redbox]   at (0,-2.5) {};
\node[redbox]   at (3,-2.5) {};
\node[redbox]   at (6,-2.5) {};
\node[redbox]   at (9,-2.5) {};
\node[greenbox] at (12,-2.5) {};
\node[greenbox] at (15,-2.5) {};
\node[redbox]   at (18,-2.5) {};

\node at (-2,-5) {\Large $\hat{p}$};
\node[redbox]    at (0,-5) {};
\node[bluebox]   at (3,-5) {};
\node[bluebox]   at (6,-5) {};
\node[redbox]    at (9,-5) {};
\node[greenbox]  at (12,-5) {};
\node[greenbox]  at (15,-5) {};
\node[redbox]    at (18,-5) {};

\begin{scope}[shift={(19.8,0.8)}]
    \node[
        draw=darkgray,
        rounded corners=0.5pt,
        fill=red!20,
        pattern=north east lines,
        pattern color=red!70!black,
        minimum width=0.35cm,
        minimum height=0.35cm,
        inner sep=0pt,
        line width=0.3pt
    ] at (0,0) {};
    \node[anchor=west, font=\scriptsize] at (0.3,0) {$p_b$};

    \node[
        draw=darkgray,
        rounded corners=0.5pt,
        fill=blue!20,
        pattern=dots,
        pattern color=blue!70!black,
        minimum width=0.35cm,
        minimum height=0.35cm,
        inner sep=0pt,
        line width=0.3pt
    ] at (0,-0.45) {};
    \node[anchor=west, font=\scriptsize] at (0.3,-0.45) {$p_1$};

    \node[
        draw=darkgray,
        rounded corners=0.5pt,
        fill=green!20,
        pattern=crosshatch,
        pattern color=green!70!black,
        minimum width=0.35cm,
        minimum height=0.35cm,
        inner sep=0pt,
        line width=0.3pt
    ] at (0,-0.9) {};
    \node[anchor=west, font=\scriptsize] at (0.3,-0.9) {$p_2$};
\end{scope}
\end{tikzpicture}
    \caption{Illustration of two specialized models $p_1$ and $p_2$ satisfying factorized conditionals (Definition~\ref{def:fac_con}) and their composition $\hat p$ under Equation~\ref{eq:comp_op_auto}. Each column represents a token coordinate in the generated sequence. Blue cells mark the active mask $M_1$ where $p_1$ differs from the background model $p_b$, green cells mark the active mask $M_2$ where $p_2$ differs from $p_b$, and red cells indicate coordinates on which the corresponding specialized model agrees with $p_b$. Because the active masks are disjoint, the likelihood-ratio composition cancels the inactive factors and yields a composed model that follows $p_1$ on $M_1$, follows $p_2$ on $M_2$, and follows $p_b$ elsewhere, giving the projective behavior stated in Theorem~\ref{thm:main}.}
    \label{fig:fac_con}
\end{figure}

\section{Acknowledgments}
    \label{app:ack}
    This research has been supported by the French government National Research Agency (ANR) through the UCA JEDI (ANR-15-IDEX-01), EUR DS4H (ANR-17-EURE-004), and the 3IA Côte d’Azur Investments in the Future project with the reference number ANR-23-IACL-0001. A huge thanks to Matteo Stromieri for helping us with the experiments. Experiments presented in this paper were carried out using the Grid'5000 testbed, supported by a scientific interest group hosted by Inria and including CNRS, RENATER and several Universities as well as other organizations (see https://www.grid5000.fr).
    


\end{document}